\documentclass[preprint,hidelinks,onefignum,onetabnum]{siamart220329}

\usepackage{lipsum}
\usepackage{amsfonts}
\usepackage{graphicx}
\usepackage{epstopdf}
\usepackage{algorithmic}
\usepackage{mathrsfs}
\usepackage{bm}
\ifpdf
  \DeclareGraphicsExtensions{.eps,.pdf,.png,.jpg}
\else
  \DeclareGraphicsExtensions{.eps}
\fi

\usepackage{enumitem}
\setlist[enumerate]{leftmargin=.5in}
\setlist[itemize]{leftmargin=.5in}

\newsiamremark{remark}{Remark}
\newsiamremark{hypothesis}{Hypothesis}
\crefname{hypothesis}{Hypothesis}{Hypotheses}
\newsiamthm{claim}{Claim}
\newsiamthm{assumption}{Assumption}
\newsiamthm{example}{Example}

\headers{Interpolation, Approximation and Controllability of DNNs}{J. Cheng, Q. Li, T. Lin, and Z. Shen}

\title{Interpolation, Approximation and Controllability of Deep Neural Networks
}

\author{
    Jingpu Cheng\thanks{Department of Mathematics, National University of Singapore, 117543, Singapore
    (\email{chengjingpu@u.nus.edu})}
    \and
    Qianxiao Li\thanks{Department of Mathematics and Institute for Functional Intelligent Materials, National University of Singapore, 117543, Singapore (\email{qianxiao@nus.edu.sg})}
    \and
    Ting Lin\footnotemark[3]\thanks{School of Mathematical Sciences, Peking University, 100871, China
    (\email{lintingsms@pku.edu.cn})}
    \and Zuowei Shen\thanks{Department of Mathematics, National University of Singapore, 117543, Singapore
    (\email{matzuows@nus.edu.sg})}
}

\usepackage{amsopn}
\DeclareMathOperator{\supp}{supp}
\DeclareMathOperator{\curl}{curl}

\ifpdf
\hypersetup{
  pdftitle={Interpolation, Approximation and Controllability of Deep Neural Networks},
  pdfauthor={J. Cheng, Q. Li, T. Lin, and Z. Shen}
}
\fi

\begin{document}

\maketitle

\begin{abstract}
  We investigate the expressive power of deep residual neural networks idealized as continuous dynamical systems through control theory.  Specifically, we consider two properties that arise from supervised learning, namely universal interpolation - the ability to match arbitrary input and target training samples - and the closely related notion of universal approximation - the ability to approximate input-target functional relationships via flow maps.  Under the assumption of affine invariance of the control family, we give a characterisation of universal interpolation, showing that it holds for essentially any architecture with non-linearity.  Furthermore, we elucidate the relationship between universal interpolation and universal approximation in the context of general control systems, showing that the two properties cannot be deduced from each other. At the same time, we identify conditions on the control family and the target function that ensures the equivalence of the two notions.
\end{abstract}

\begin{keywords}
deep neural networks, controllability, universal interpolation, universal approximation
\end{keywords}

\begin{MSCcodes}
93B05, 41A05, 68T07
\end{MSCcodes}

\section{Introduction}
Deep neural networks have emerged as powerful tools in various domains,
demonstrating remarkable success in tasks such as image classification and
natural language processing.
The central advantage of these deep architectures over their shallow counterparts
is their ability to utilize function composition through the stacking of layers.
Consequently, a fundamental problem in the theory of deep learning
is to understand the expressive power of deep neural networks generated by the composition of
relatively simple functions. Previous research has investigated this topic
through various approaches, including the demonstration of their ability to
approximate function families known to have strong expressive capacity (e.g.,
polynomials, wavelets)~\cite{guhring2020ErrorBoundsApproximations,kidger2020UniversalApproximationDeep,leshno1993MultilayerFeedforwardNetworksa,mallat2016UnderstandingDeepConvolutional}, explicit constructions based on specific network
structures~\cite{lin2018ResNetOneneuronHidden,lu2021DeepNetworkApproximationa,schwab2023DeepLearningHigha,shen2020DeepNetworkApproximation,shen2021DeepNetworkApproximation,yarotsky2017ErrorBoundsApproximationsa},
and analysis from the perspective of dynamical systems and
control theory~\cite{li2022DeepLearningDynamical,ruiz-balet2022InterpolationApproximationMomentum,ruiz-balet2023NeuralODEControl,tabuada2020UniversalApproximationPower}.
Many studies rely on specific structure
of activation functions or require increasing the width
in the hidden layers beyond the input dimension.
Yet, in order to gain a mathematical understanding of deep learning,
it is important to isolate the effect of composition
and study how it may be used to build expressivity in a general setting,
without further requirements on width and specific architectural choices.

One of the most popular deep architectures in machine learning is the residual
neural network (ResNet)~\cite{he2016DeepResidualLearninga}.
A dense variant of the ResNet
(where the input dimension and output dimension are the same) with $S$ layers
generates the output map $x_0\to x_S$ by iterating the following difference
equation:
\begin{equation}
    \label{eq:resnet}
x_{s+1}=x_s+h W_s \sigma\left(A_s x_s+b_s\right),\quad  s=0,1,\cdots, S-1,
\end{equation}
where $h>0$, $x_s,b_s\in \mathbb R^d$, $W_s, A_s\in \mathbb{R}^{d\times d}$ and $\sigma$ represents the non-linear activation function.
The ResNet architecture is a repeated composition of simple shallow neural network layers,
and its practical success poses the question of how the composition of functions enhances expressive power.
The key observation~\cite{e2017ProposalMachineLearning,li2018MaximumPrincipleBased,ruthotto2020DeepNeuralNetworks} is that this type of ResNets can be regarded as a forward Euler time stepping of a (parameterized) continuous-time dynamical system:
\begin{equation}
    \label{eq:ct-resnet}\dot{x}(t)=W(t) \sigma(A(t) x(t)+b(t)),\quad t\in [0,T],\ x(0)=x_0,
\end{equation}
and function composition can now be understood as time-evolution.
In what follows, we refer to \eqref{eq:ct-resnet} as a continuous-time ResNet,
which is also known as a type of ``neural ODEs'' in the machine learning literature~\cite{chen2018NeuralOrdinaryDifferential}.
It is important to note that the system \eqref{eq:ct-resnet} is automatically a
continuous-time control system, with controls being all the parameters $W(t), A(t),b(t)$.
Therefore, continuous-time idealization enables us to leverage dynamical
systems and control theory in the study of deep learning.
This led to progress in approximation theory~\cite{kang2022FeedforwardNeuralNetworksa,li2022DeepLearningDynamical,ruiz-balet2023NeuralODEControl,tabuada2020UniversalApproximationPower}, training algorithms~\cite{dong2020AdaptiveResidualNetwork,e2019MeanFieldOptimalControl,li2018MaximumPrincipleBased},
adversarial defence~\cite{chen2022.SelfHealingRobustNeural,chen2020.RobustNeuralNetworks} and generative modelling~\cite{chen2018NeuralOrdinaryDifferential,grathwohl2018FFJORDFreeFormContinuous}.

In the supervised learning setting, two related questions can be posed
in this dynamical view.
The first is the \emph{universal approximation property} (UAP), which represents the
ability of a control system to approximate input-output functional relationships via its
flow maps. The universal approximation of continuous-time ResNets has been
studied in previous works. For example, in~\cite{li2022DeepLearningDynamical},
it is shown that under
mild conditions, the flow map of a large class of control systems, including continuous-time ResNets,
can be arbitrarily close (in $L^p$ sense) to any continuous function on a prescribed compact domain $K$.
Alternatively, the ability to reduce empirical loss is another measure of model capacity.
This points to the notion of \emph{universal interpolation property} (UIP),
which characterizes a model's capacity to match arbitrary input and target
training samples. For control systems, the universal interpolation property can be
formulated as a controllability problem of steering an arbitrary ensemble of points to their
prescribed ensemble of targets simultaneously.
In this direction, the UIP has been established for certain control-affine systems in~\cite{agrachev2022ControlManifoldsMappings,cuchiero2020DeepNeuralNetworks} and continuous-time ResNets with different activation functions in~\cite{ruiz-balet2023NeuralODEControl,ruthotto2020DeepNeuralNetworks,tabuada2020UniversalApproximationPower}. These studies typically analyze the expressive
power by investigating the controllability over
ensembles of points, demonstrating that a variety of seemingly simple control systems may
exhibit strong expressive power in interpolation and approximation tasks.
However, current studies are limited to specific architectures,
and a general characterization of universal interpolation is lacking.
Moreover, it is not clear if and when universal interpolation leads to universal approximation and vice versa.
This is an important question as it sets forth the extent to which control theory
can be a powerful analytical tool for understanding the distinguishing
feature of deep learning, namely composition idealized as dynamics.

In this paper, we establish some primary results on the interpolation and approximation power
of continuous-time control systems.
While a general characterization of universal interpolation (i.e., ensemble controllability~\cite{agrachev2020ControlSpacesEnsembles})
is challenging,
we exploit the fact that most practical deep learning architectures lead to control
families that satisfy affine invariance.
For example, the control family corresponding to continuous-time ResNets
\begin{equation}
    \label{eq:F_resnet}
\mathcal F_{ResNet}:=\{W\sigma(A\cdot+b)\mid W,A\in\mathbb R^{d\times d},b\in\mathbb R^d\}
\end{equation}
is affine invariant, meaning it remains closed under any affine transformation
$f\mapsto Wf(A\cdot+b)$.
This affine invariance property, commonly found in
neural network structures but rarely considered in classical control problems,
provides the system with an infinite-dimensional space of state equations,
and significantly weakens the conditions required for UIP and UAP to hold.
Under the assumption of affine invariance,
our first result (Theorem~\ref{thm:con_aff}) gives a
characterisation of UIP for general control systems,
i.e., UIP holds if and only if the control family contains a non-linear function.
When applied to ResNets, our result (together with known results in~\cite{li2022DeepLearningDynamical,ruiz-balet2023NeuralODEControl})
shows that UAP also holds for any such architecture with a non-linear Lipschitz
activation function.
This finding relaxes the assumptions on non-linearity originally presented
in~\cite{li2022DeepLearningDynamical}, albeit with stronger affine invariance requirements.

The successful deduction of UAP from UIP in continuous-time ResNets has
motivated us to explore the relation between these two concepts in general.
Our second result (Theorem~\ref{thm:uap_uip})
demonstrates through construction that for continuous-time control
systems, UAP and UIP cannot be deduced from each other.
Consequently, one cannot generally obtain UAP from establishing UIP.
Nonetheless, as shown in Proposition~\ref{prop:uni_time}, we can
still identify conditions on the control family and the target function space
under which a connection between them can be established.

\section{Formulation and Results}
In this section, we first introduce the formulation of approximation and
interpolation problem using flow maps of control systems. Subsequently, we
present our main results.

\subsection{Control system and flow map}
Consider the following parameterized dynamical system
\begin{equation}
\label{eq:con_sys}
\dot x(t) = f(x(t);\theta(t)),\quad x(0)=x_0,\quad\theta(t) \in\Theta, \quad\forall t \in [0,T],
\end{equation}
where $x(t)\in\mathbb{R}^d$ and $\theta(t)$ are called
the state and control, respectively.
To simplify the presentation,
throughout this paper we assume that the admissible control set $\Theta \subseteq \mathbb{R}^l$
is independent of time.
Furthermore, we focus on
the case that the control function $t\mapsto \theta(t)$ is piece-wise constant,
rather than Borel measurable, as is commonly assumed.
This is sufficient to establish approximation and interpolation results.

Next, we introduce the set of flow maps of control systems, which serves as a dynamical
hypothesis space for continuous-time control systems. In the following, we
assume that $x\mapsto f(x;\theta)$ is globally Lipschitz for any
control $\theta \in \Theta$.
We denote by $\mathcal F$ the \emph{control family}, i.e., the
parameterized family of all possible control functions in $x$:
\begin{equation}
    \mathcal F:=\{x\mapsto f(x;\theta)\mid \theta \in \Theta \}.
\end{equation}
Classical results~\cite{arnold1992ordinary} indicate that for any $g\in\mathcal F$ and $t>0$,
the initial value problem
\begin{equation}
    \dot x(t)=g(x(t)), \quad x(0)=x_0
\end{equation}
is well-posed.
This allows us to define $\varphi_{t}^g$ as the \emph{flow map}, or simply flow, $x_0\mapsto x(t)$.
Subsequently, the set of flow maps of system~\eqref{eq:con_sys} with piece-wise constant controls
at time $T>0$ can be defined as
\begin{equation}
        \Phi(\mathcal{F}, T):=\left\{\varphi_{t_k}^{f_k}\circ \varphi_{t_{k-1}}^{f_{k-1}}\circ \cdots\circ \varphi_{t_1}^{f_1}(\cdot) \mid t_1+\cdots t_k=T,f_1,\cdots,f_k\in \mathcal F , k\ge 1\right\},
\end{equation}
where $\varphi_{t_k}^{f_k}\circ \varphi_{t_{k-1}}^{f_{k-1}}\circ \cdots\circ \varphi_{t_1}^{f_1}(\cdot)$
is the flow map generated by the control $\theta(t)$ with
\begin{equation}
    f(\cdot,\theta(t)):=f_j,\quad \text{ for } t \text{ such that }\sum_{i=1}^{j-1}t_i\le t<\sum_{i=1}^{j}t_{i}.
\end{equation}

We denote $\mathcal{A}_{\mathcal F,T}$ as the family of all flows within time
$T$, and denote $A_{\mathcal F}$ as the set of all the flows in finite time
horizon:
\begin{equation}
    \mathcal A_{\mathcal F,T}:=\bigcup_{0\le t\le T}\Phi(\mathcal{F}, t),
\end{equation}
\begin{equation}
    \mathcal A_{\mathcal F}:=\bigcup_{T\ge 0}\mathcal A_{\mathcal F,T}.
\end{equation}
In~\cite{li2022DeepLearningDynamical}, the mapping family $\mathcal A_{\mathcal F,T}$ and $\mathcal A_{\mathcal F}$
are called the \emph{attainable set} of time horizon $T$ and the total attainable set
of $\mathcal F$, respectively. Note again that a similar definition has been
introduced in classical control theory, cf.~\cite[Chap. 5]{agrachev2004control}.
The key difference is that the attainable set therein is a subset of $\mathbb{R}^d$,
consisting of all the positions that a single point
$x_0$ can be steered to.
In our case, the attainable set is a family of mappings $\mathbb R^d\to\mathbb R^d$,
and the classical attainable set with respect to the point $x_0$ is
$\{\varphi(x_0) : \varphi \in \mathcal A_{\mathcal F,T}\}$ and $\{\varphi(x_0) : \varphi \in
\mathcal A_{\mathcal F}\}$, respectively.

\begin{remark}
    Some non-globally-Lipschitz
    control functions, such as polynomials, are also of interest in our setting.
    In fact, the concept of flow maps and attainable sets can also be extended to
    the case when the control function is only locally Lipschitz in $x$ by adding
    some restrictions on the control function $\theta(t)$. We will discuss this in
    ~\Cref{subsubsec:loc_lip}.
\end{remark}

\subsection{Universal interpolation property and universal approximation property}

In this subsection, we introduce the precise definitions of the \emph{universal
interpolation property} (UIP, as introduced in~\cite{cuchiero2020DeepNeuralNetworks}).
and the \emph{universal approximation property} (UAP) for
control systems.
Through the dynamical view of deep learning, the UIP corresponds to
the ability of a ResNet-type architecture to achieve arbitrarily small training error
on any finite dataset, whereas the UAP refers to its ability
to approximate an input-output relationship, possibly on a compact domain
in the input (or feature) space.
Both properties represent some form of expressiveness of deep neural networks
idealized as control systems.
For clarity, we first introduce these notions when the
dimension of the system is greater than 1 (almost always the case in applications)
to avoid topological restrictions of continuous-time dynamics in one dimension
(see \Cref{subsubsec:1dsys}).

We begin with the UIP.
Intuitively, a control system possessing UIP can interpolate, using its flow maps,
any finite number of data to arbitrary precision.
\begin{definition}[UIP for $d\ge 2$]
    \label{def:uip}
    For $d\ge 2$, we say a control system driven by a control family $\mathcal F$
    has the approximate universal interpolation property
    if for any $\varepsilon > 0$, any positive integer $N$,
    and any $N$ data points $(x_1,y_1),\cdots, (x_N,y_N)$
    with $x_i \neq x_j$, $y_i \neq y_j$ for all $i\neq j$,
    there exists $\varphi(\cdot)\in \mathcal A_{\mathcal F}$ such that
\begin{equation}
    \label{eq:uip_def}
    \Vert \varphi(x_i)-y_i \Vert_\infty\le\varepsilon,\quad i=1,\cdots, N.
\end{equation}
If ~\eqref{eq:uip_def} also holds for $\varepsilon=0$, we say the system possesses the exact universal interpolation property.
\end{definition}


A closely related notion is the UAP,
which refers to the density of the attainable set in some appropriate function space.
This is a familiar notion in approximation theory~\cite{devore1993constructive} and learning theory~\cite{hornik1989MultilayerFeedforwardNetworks}.
Here we focus on the $L^p$ variant of this definition,
which corresponds to density in the topology of compact convergence in $L^p$.

\begin{definition}[UAP for $d\ge 2$]
    \label{def:uap}
    We say a $d$-dimensional control system with control family $\mathcal F$ has the universal approximation property in $L^p$ sense
    ($1\le p<\infty$), if for any continuous function $F: \mathbb R^d\to \mathbb R^d$ and compact set $K\subseteq \mathbb R^d$ and $\varepsilon>0$, there exists $\varphi\in\mathcal A_{\mathcal F}$ such that
    \begin{equation}
        \Vert F-\varphi\Vert_{L^p(K)}<\varepsilon.
    \end{equation}
\end{definition}

A system that exhibits the UAP possesses strong expressive power,
in that it is capable of approximating any reasonable input-output relationship
on compact input domains.
This is a desirable property to achieve for any neural network architecture design process,
and is a basic guarantee of its general applicability.

The reason for the restriction $d\geq 2$ in the above definitions is the so-called \textit{topological constraint}, as mentioned in~\cite{li2022DeepLearningDynamical}.  Namely, for all
one-dimensional control systems, the resulting flow map are continuously
increasing functions. As a consequence, all the attainable sets aforementioned can only approximate or interpolate increasing functions.
Thus, in one dimension this additional constraint is imposed to the respective definitions
of UAP and UIP. See~\Cref{subsubsec:1dsys} for detailed discussions.

\begin{remark}
    Since the flow maps are all invertible, they cannot map distinct points to
    the same image. Thus, in the definition of UIP, we require the images of
    distinct points to be distinct. In general, control systems will not
    possess UAP in $L^\infty$ sense.
    This can be shown by a simple modification on the example
    provided in~\cite[Sec. 4]{dupont2019AugmentedNeuralODEsa}.
\end{remark}

We close this subsection with the following observation. From
Definition~\ref{def:uip}, the exact UIP can be formulated as a
controllability problem.
The ability of a control system's flow map to interpolate a
finite dataset $X=\{(x_i,y_i)\}_{i=1}^N$ is equivalent to its ability to steer the initial points $x_i$ towards their respective targets $y_i$, simultaneously and with the same control.

This can also be formulated as a classical controllability problem in ($Nd$)-dimensions
with tensor-type control families~\cite{agrachev2020ControlSpacesEnsembles,cuchiero2020DeepNeuralNetworks,ruiz-balet2023NeuralODEControl,tabuada2020UniversalApproximationPower}
\begin{equation}
    \label{eq:Nd_system1}
    \dot{X}_N(t)=\Big(f(x_1(t);{\theta}(t)), f(x_2(t);{\theta}(t)), \cdots , f(x_N(t);{\theta}(t)) \Big), \quad \forall t \in [0,T],
\end{equation}
where $X_N(t)=(x_1(t),\cdots,x_N(t))\in\Omega_{N}\subset \mathbb R^{Nd}$, with
\begin{equation}
    \Omega_N=\Big\{(x_1,\cdots,x_N)\mid \forall i\neq j ,x_i\neq x_j, i,j\in\{1,2\cdots, n\} \Big\}.
\end{equation}
Thus, a system has the exact UIP if and only if the family of control systems
~\eqref{eq:Nd_system1} are controllable for all positive integers $N$.

While this formulation seemingly translates
UIP to a classical controllability problem,
we should note that there is an essential difference.
Here, we have not one but a countable family of controllability
problems, since controllability needs to hold for all positive integers $N$.
Thus, classical sufficient conditions for controllability (e.g. Chow--Rashevsky theorem~\cite[Theorem 5.9]{agrachev2004control})
needs to be checked for all $N >0 $, which is a challenging task.
Note that there exist controllability results dealing with infinite number of points, e.g., the
Chow--Rashevsky theorem into a Banach manifold~\cite{agrachev2022ControlManifoldsMappings}.
However, as the Lie algebra generated by general non-linear vector fields can be
very complicated and Chow--Rashevsky theorem is just a sufficient condition for
general smooth systems, a characterization of UIP cannot be readily derived from these
results directly.

In the literature, there are constructions of relatively simple
control systems possessing UIP.
An example is the following control-affine type systems
\begin{equation}
    \label{eq:aff_sys1}
    \dot{\bm x}(t) = u_1 f_1(\bm x(t)) + u_2 f_2(\bm x(t)) + \cdots u_k f_k(\bm x(t)).
\end{equation}
In~\cite{cuchiero2020DeepNeuralNetworks}, it is
proved that there exist five control functions $f_i$ to achieve UIP with a non-constructive argument.
However, for given right-hand side $f_1,\cdots, f_k$, it is in general difficult to show whether the system has UIP.
Rather than following these approaches, here we consider restricting the control family
to those having \emph{affine invariance}.

\begin{definition} [Affine invariance]
    \label{def:aff}
    Let $\mathcal{F}$ be a set of functions from $\mathbb{R}^d$ to $\mathbb{R}^d$. We say that $\mathcal{F}$ is affine invariant if $f \in \mathcal{F}$ implies $W f(A \cdot-b) \in \mathcal{F}$ for any $b \in \mathbb{R}^d$ and $A,W\in\mathbb R^{d\times d}$.
\end{definition}
One can check from \eqref{eq:F_resnet} that affine invariance holds for control families corresponding to ResNets.
Affine invariance arises naturally
from the practical architecture of deep neural networks, but is rarely
considered in classical controllability analysis.
In contrast to classical control-affine systems, an affine invariant family
comprising non-linear maps has the capability to generate an infinite
dimensional space. This property simplifies the characterization of UIP
and forms the basis of our analyses.

\subsection{Main results}


\subsubsection{characterisation of UIP}
Our first result gives a characterisation for control system with UIP under the affine invariance assumption on its control family.
\begin{theorem}
    \label{thm:con_aff}
    Let $\mathcal F$ be an affine invariant control family.
    Then, the control system with control family $\mathcal F$ possesses the exact UIP
    if and only if there exists $f=(f_1,\cdots,f_d)\in\mathcal F$ with at least one component $f_j$
    being non-linear.
\end{theorem}

A direct consequence of Theorem~\ref{thm:con_aff}
is that any continuous-time ResNet, following the structure of
\begin{equation}
\dot{x}(t)=W(t) \sigma(A(t) x(t)+b(t)),
\end{equation}
where $\sigma:\mathbb R\to \mathbb R$ is any
non-linear Lipschitz activation function applied element-wise to a vector,
possesses the exact UIP.
Moreover, by leveraging the ``shrinking map''
technique employed in prior works such as~\cite{li2022DeepLearningDynamical,ruiz-balet2023NeuralODEControl,tabuada2020UniversalApproximationPower},
the UAP also holds for such continuous-time ResNets
(See Corollary~\ref{cor:resnet_uap} for details).
Consequently, ResNets (both continuous-time and discrete) can achieve UAP
even if the width of each layer is bounded,
provided that it is at least the input dimension $d$.
This covers almost all the activation functions used in practice, e.g.
ReLU: $\sigma(x)=\max\{x,0\}$, Leaky ReLU: $\sigma(x)=\max\{ax,x\}\,(0<a<1)$,
Sigmoid: $\sigma(x)=({1+e^{-x}})^{-1}$,
and Tanh: $\sigma(x)=\tanh{(x)}$.

So far, our focus has been on the expressive power of ResNets whose input and
output dimensions are the same. However, practical network architectures often
incorporate an output map $g$ from a designated family $\mathcal G$ of functions
that map from $\mathbb R^d$ to $\mathbb R^m$, to match the input and output
dimensions. $\mathcal G$ is usually called the terminal family and is typically
very simple, e.g. a family of affine maps. In this context, and in combination
with Proposition 3.8 from~\cite{li2022DeepLearningDynamical}, a corollary of
Theorem~\ref{thm:uap_uip} is the following result for approximation of functions
from $\mathbb R^d \to \mathbb R^m$.

\begin{corollary}
    Let $\mathcal A_{\mathcal F}$ be the attainable set of a continuous-time  ResNet with a non-linear activation function
    and $\mathcal G$ is a
    terminal family consisting of Lipschitz functions.
    For given continuous function $F:\mathbb R^d\to\mathbb R^m$, suppose that for
    any compact set $K\subseteq \mathbb R^d$, there exists $g\in\mathcal G$ such
    that $F(K)\subset g(\mathbb R^d)$. Then, for any $\varepsilon>0$ and $p\in
    [1,+\infty]$, there exists $\varphi\in\mathcal A_F$ and $g\in\mathcal G$ such
    that

    \begin{equation}
        \Vert F-g\circ\varphi\Vert_{L^p(K)} <\varepsilon
    \end{equation}
\end{corollary}


Now, let us discuss the implications of Theorem~\ref{thm:con_aff} in both deep
learning and control theory.
From the perspective of deep learning, our result
implies that function composition and affine invariant families enable
interpolation for almost any non-linear residual architecture.
Previous works have established a number of sufficient conditions for the
universal interpolation property (UIP) while studying the expressive ability of
continuous-time ResNets. For example, in~\cite{ruiz-balet2023NeuralODEControl},
UIP was established for continuous-time ResNets with the ReLU type activation
function, and in~\cite{tabuada2020UniversalApproximationPower}, UIP was
established in the case when the activation function satisfies a quadratic ODE.
The construction in~\cite{ruiz-balet2023NeuralODEControl}
relied on a degree of affine invariance.
In contrast,~\cite{tabuada2020UniversalApproximationPower} considered a finite subset of $\mathcal F$
to check the Lie bracket generating conditions and do not explicitly
utilize affine invariance when establishing the UIP.
In~\cite{li2022DeepLearningDynamical}, the authors considered a weaker form of
affine invariance compared to what is considered in this paper. They provided a
sufficient condition for UIP and UAP, which covers cases of the activation
function involving ReLU, Sigmoid, and Tanh. Additionally, UIP was established
for ResNets with increased width in~\cite{ruiz-balet2022InterpolationApproximationMomentum}.
In comparison to these sufficient conditions or explicit constructions,
Theorem~\ref{thm:con_aff} offers a characterisation of UIP
for all affine invariant control systems (thus all such ResNet-type architectures).
This theorem demonstrates that for neural networks with sufficiently large depth but fixed width,
the interpolation power can be guaranteed by non-linear activation functions.

Shifting to a control perspective, UIP equates to a set of simultaneous control
problems. Our result highlights the benefits of affine invariance, a
characteristic of neural networks, in addressing such problems.
In the control literature, simultaneous control of a finite family
of linear systems in the form has been studied for many systems
(see~\cite{lions1988ExactControllabilityStabilization,loheac2016AveragedSimultaneousControllabilitya,tucsnak2000SimultaneousExactControllability} and references in~\cite[Sec. 1]{ruiz-balet2023NeuralODEControl}).
However, the UIP necessitates a greater level of controllability than
is previously investigated in control theory:
the control system is required to be the same for all initial data,
and yet be able to simultaneously control an arbitrary
number of points.
In particular, this requirement has exceeded the capabilities of linear systems,
as their flow maps are always linear,
thus it cannot control an arbitrarily large number of points.
For non-linear control-affine systems in the form~\eqref{eq:aff_sys1},
sufficient conditions for and explicit constructions of
control systems with UIP are derived~\cite{agrachev2022ControlManifoldsMappings,cuchiero2020DeepNeuralNetworks}.
These results hint at the possibility that UIP may not be hard to achieve,
but is in general not easy to check.
In this sense, our result shows that under affine-invariance,
a simple characterisation of UIP can be obtained.

In practical scenarios, research interests extend to neural networks with
constrained architectures, such as Convolutional Neural Networks (CNNs). Unlike
fully connected ResNets, CNNs utilize weight sharing and do not satisfy the
full affine invariance assumption introduced in Definition~\ref{def:aff}. This motivates us to study
UIP under weaker affine invariance conditions. In other words, we impose
restrictions on the transformation matrices $W$ and $A$ in Definition~\ref{def:aff},
limiting them to a subset of $\mathbb R^{d\times d}$.  By
following our methodology, we can derive characterizations, or at least
easily checkable sufficient conditions, for UIP under weaker affine invariance assumptions.
These results are discussed in detail in \Cref{subsubsec:weaker_aff}.

Let us summarise briefly the proof techniques.
We begin the proof of Theorem~\ref{thm:con_aff} and also the results of UIP under
weaker affine invariance assumptions (\Cref{subsubsec:weaker_aff}) by noting that if UIP fails, an identity
emerges that is valid for all functions $f\in \mathcal F$ and transformation
parameters $W,A,b$. A key step in our technique involves taking the Fourier
transform with respect to the shift parameter
$b$.
Integrating this with admissible choices of parameters
$W$ and $A$ results in a constraint on the support of Fourier transform of $f$.
When there are sufficient varied choices of
$W$ and $A$ due to some degree of affine invariance,
this constraint can become strict, thereby giving the desired sufficient or
necessary conditions.

\subsubsection{The gap between UAP and UIP}
Our second result explores the relation between UIP and UAP in the context of
control systems. Intuitively, UIP and UAP are two closely related concepts.
UIP concerns the ability of a control system to steer an arbitrarily large number of points together,
whereas UAP refers to the density of the attainable set consisting of flow maps of the control system.
Approximation theory in the deep learning literature tend to focus on UAP,
whereas control perspectives often discuss UIP.
Due to their close intuitive meaning, one may expect that they are similar notions,
but a systematic analysis of their relationship is lacking.
This is an important question, because it underscores
the extent of the applicability of control-theoretic methods
to understand approximation theory of deep learning.

Previous works on the interface of control and deep learning deduce UAP from UIP
using a ``shrinking map'' technique~\cite{li2022DeepLearningDynamical,ruiz-balet2023NeuralODEControl,tabuada2020UniversalApproximationPower}.
However, this is only a sufficient condition, and relies on the specific structure of ResNet-type control families.
Our aim here is to investigate this question in greater generality, namely:
\begin{quote}
    \centering\it
        Is there a gap between UIP and UAP for general control systems?
\end{quote}

Let us first consider this problem in the context of classical approximation theory,
i.e., approximating target functions using linear combinations of simple basis functions.
In this case, the hypothesis space $\mathcal H$ is a linear space. Let us assume that the target space
$\mathcal B$ is a separable Banach space. Clearly, for $\mathcal H
\subseteq \mathcal B$, $\mathcal H$ is not dense in $\mathcal B$ if and only if
there exists a non-zero bounded linear functional $L$ such that $\mathcal H
\subseteq \ker L$. The functional $L$ limits the expressive power of $\mathcal
H$, but does not necessarily prevent $\mathcal H$ from interpolating all finite
sets of data points.

\begin{example}
\label{exa:poly}

Specifically, now we take $\mathcal B = C([a,b])$, the space of all continuous
function on the interval $[a,b]$, with $b > a > 0$.  Define $\mathcal P :=
\operatorname*{span} \{ x^n \mid n \in \mathbb N_{\ge 0}\}$, the set of all
polynomial functions on $[a,b]$. Consider $\mathcal P_0 := \{p \in \mathcal P
\mid \int_{a}^b p\, d x = 0 \}$. Clearly, as
the kernel of the linear functional $L(f)=\int_{a}^{b} f \, d x$, $\mathcal P_0$
cannot approximate any continuous
functions with non-zero integral on $[a,b]$. However, it is not hard to see that
$\mathcal P_0$ can still interpolate any finite data points.
Another example is $\mathcal P_{sq} := \operatorname{span} \{ x^{n^2} \mid n \in \mathbb N_{\ge
0}\}$. It is not difficult to see that $\mathcal P_{sq}$ can interpolate any
finite set of data points, but the celebrated M\"untz--Sz\'asz theorem (see e.g.
\cite[Chap. 11]{devore1993constructive}) shows that the space $\mathcal P_{sq}$ is not
dense in $\mathcal B$. This implies that there exists a non-zero bounded linear
functional $L$, such that $L(p) = 0 $ for all $p \in P_{sq}$, but the explicit form
of $L$ may be complicated.
\end{example}

\Cref{exa:poly} demonstrates that in the context of linear approximation, UIP and UAP
have a clear gap, at least in separable Banach spaces.
Furthermore, identifying a ``conserved quantity'' in the form of the linear functional $L$
helps us identify cases where UIP holds but UAP does not. On the other hand, if we consider approximation in $L^p$ sense, it is easy to identify cases where UIP fails but UAP holds. For example, the set of polynomials with zero constant term is dense in $L^p([0,1])$, but cannot match the input $0$ to the label $1$.

In the context of control systems, the problem becomes quite different.
The hypothesis space $\mathcal A_{\mathcal F}$ does not even have a linear structure but rather possesses a compositional structure.
However, the idea of constructing ``conserved quantities'' and constraining the hypothesis to satisfy the conservation law
can also be used to show that UIP does not imply UAP in general control systems.
On the other hand, as the UAP is defined in $L^p$ sense,
we can also find examples that UAP holds but UIP fails.
Compared with the linear case,
the example will be less trivial as the UAP
is harder to establish without the linear structure.
This leads to the following result:
\begin{theorem}
    \label{thm:uap_uip}
        UIP and UAP are not equivalent for control systems in general. Concretely,
    \begin{enumerate}
        \item there exists a control system that possesses exact UIP, but does not possess UAP;
        \item there exists a control system that possesses UAP, but does not possess exact UIP.
    \end{enumerate}
\end{theorem}

Let us discuss the key insights to the constructions that prove Theorem~\ref{thm:uap_uip}.
For the first direction, we construct a control system where the flow
is volume-preserving for any control parameter.
Concretely, we consider the dynamics
\begin{equation}
    \left\{
    \begin{aligned}
    &\dot x_1=-\theta_1-2\theta_3 x_1^2x_2,\\
    &\dot x_2=\theta_2+2\theta_3 x_1x_2^2,
    \end{aligned}
    \right.
\end{equation}
where $x=(x_1,x_2)\in\mathbb R^2$ is the state and
$\theta=(\theta_1,\theta_2,\theta_3)\in\mathbb R^3$ is the control.
One may check that the Jacobian of any function in the attainable set must have unit determinant,
therefore it does not possess UAP.
This is the ``conservation law''.
However, one may check explicitly using the Chow-Rashevsky theorem that it possesses UIP.
The details are found in~\Cref{subsubsec:uip_not_uap}.

Conversely, we notice that UAP only
addresses approximation of functions in an average sense, thus it may not necessarily imply exact
interpolation on finite data. This is to say, a system having UAP can only
approximate the target function well roughly, but it might allow small
oscillation which violates the UIP. The introduced oscillation can be
overlooked since only the $L^p$ distance is tracked, consequently the control system
with UAP still has the freedom to lose controllability on a set of measure
zero. The example is constructed simply from polynomial control, where the
UAP of such systems is easy to establish.
Specifically, we consider
\begin{equation}
    \left\{
    \begin{aligned}
    &\dot x_1=\theta_1x_1^3+\theta_2x_1^2+\theta_3x_2,\\
    &\dot x_2=\theta_4x_2^3+\theta_5x_2^2+\theta_6x_1,
    \end{aligned}
    \right.
\end{equation}
where $x=(x_1,x_2)\in\mathbb R^2$ is the state and
$\theta=(\theta_1,\cdots,\theta_6)\in\mathbb R^6$ is the control.
Here, the point $(0,0)$ is a fixed point, hence UIP fails.
See~\Cref{subsubsec:uap_not_uip} for a detailed discussion on this example.

According to the argument above, there is a gap in general between UAP and UIP.
However, a control-theoretic approach to studying approximation theory
in deep learning often relies on a successful deduction of approximation
from interpolation, established via controllability.
Thus, it is important to obtain conditions under which
this deduction can be made.
This then identifies an arena in which control-theoretic statements
are valid characterizations of the expressiveness of a deep learning architecture.
Our final result is a step in this direction.

\begin{example}
We continue with~\Cref{exa:poly}.
We showed that UIP does not imply UAP in general due to the possible existence
of a conserved quantity. Now we take another view on the failure of UAP.

Suppose that we have a polynomial hypothesis space which possesses UIP.
Naturally, to obtain UAP we consider the interpolation of a growing sequence
of finite point sets, which becomes a dense subset in the limit.
Intuitively, UAP fails only if the interpolating polynomials have an error
bounded away from zero uniformly
outside the set of interpolation points.
This implies that the polynomials must have larger and larger oscillations.
To exclude this case, one needs to impose some additional conditions.
Suppose that $\mathcal H$ possesses UIP, and we want to approximate a $C^1$
function $f_{\ast}$ on $[0,1]$.
To curb oscillations, we need some uniformity on the
interpolating functions chosen by UIP.
Specifically, such a condition can be
\begin{quote} For any $n$ data points sampled from the graph of $f_{\ast}$,
namely, $$\{(x_1, f_{\ast}(x_1)), \cdots, (x_n, f_{\ast}(x_n))\},$$ there exists
a function $f_n \in \mathcal H$ such that $f_n(x_i) = f_{\ast}(x_i)$ for all $i
= 1,2,\cdots, n.$ Moreover, we require that $\operatorname{Lip} f_{n} \le C$,
where $C$ is independent of $n$, but depends on $f_{\ast}$.
\end{quote}
With the above additional condition on $\mathcal H$ and the target function
$f$ - a type of compatibility requirement, we can now deduce the UAP from UIP.
This is because the uniformly Lipschitz condition actually
implies that any such sequence $\{f_n\}$ is uniformly bounded and
equi-continuous. Therefore, when taking $\{x_i\}_{i=1}^{\infty}$ to be a dense
subset of $[0,1]$, by Arzel\'a-Ascoli theorem, there must exist a subsequence
$\{f_{n_k}\}\subset \mathcal H$ which converges uniformly to $f_{\ast}$.
\end{example}


This program of constraining some form of ``uniformity''
readily extends to the current dynamical hypothesis space,
except that we require it in the time domain.
First, we assume that the data set $\mathcal D=\{(x_i,y_i)\}_{i=1}^N$ are
sampled from a graph of a given function $F$ on a compact set $K$, without any
noise, and denote this by $\mathcal D\sim (F,K)$.
We say $\mathcal D=\{(x_i,y_i)\}_{i=1}^N$ can be
approximately interpolated by $\mathcal A_{\mathcal F}$, if for any
$\varepsilon>0$, there exists $T(\varepsilon)>0$ and $\varphi
\in \mathcal A_{\mathcal F,T(\varepsilon)}$ such that
$\left|\varphi\left(x_i\right)-y_i\right| \leq \varepsilon$.
The infimum of all
possible $T(\varepsilon)$, namely the minimal time that $\mathcal D$ can be interpolated with arbitrarily small error, is denoted as $T(\mathcal D)$. For a map
$F:\mathbb R^d\to\mathbb R^d$, we define $T(F,K)=\inf_{\mathcal D\sim (K,F)}
T(\mathcal D)$ as the minimal time required to interpolate any finite set of data points sampled from the graph of $F$.
If $T(F,K)<\infty$, then any finite set of data points from $F$ can be interpolated in a uniform time.
We also say a control $\mathcal F$ is uniformly bounded by a linear-growth function,
if there exist positive constants $c_1$ and $c_2$ such that $|f(x)|_1 \leq c_1 +
c_2|x|_1$ for all $f \in \mathcal{F}$ and $x\in\mathbb R^d$.
We then have the following result:
 \begin{proposition}
    \label{prop:uni_time}
    Let $\mathcal F$ be a control family that is uniformly Lipschitz and
    uniformly bounded by a linear-growth function.
    Consider a given compact set $K$, and restrict $\mathcal{A}_{\mathcal{F}}$ to be a subset of $L^p(K)$. Then, for $F\in C(K)$, we have $F\in \bar{\mathcal A}_{\mathcal F, T}$, the closure of $\mathcal A_{\mathcal F,T}$ in $L^p(K)$, if and only if $T(F,K)\le T $. In other words, on given compact set $K$, $F$ can be approximated in time $T$ if and only if any finite data of $F$ can be interpolated in time $T$.
\end{proposition}
This condition demonstrates that when the control time can be well-defined, all
the finite sets of data points from the graph of $F$ can be uniformly
interpolated within a time interval $T$ if and only if $F$ can be approximated
within the same time interval. The condition directly addresses the uniformity
problem, thereby supporting our understanding of the gap between UIP and UAP.

Let us briefly summarize our results and outlook.
Our first result (Theorem~\ref{thm:con_aff}) gives a characterisation of the UIP
under the affine invariance assumption, implying that all such non-linear systems
(including all continuous-time ResNet-type architectures) possess the ability to match
arbitrary training samples.
The second result (Theorem~\ref{thm:uap_uip}) elucidates
the relationship between the UIP and UAP in the context of general control
systems, demonstrating that these two concepts cannot be deduced from each
other. The gap from the UIP to UAP can be understood as a loss of uniformity of
the control over ensembles of points.
This understanding leads to Proposition~\ref{prop:uni_time}, which
establishes a
condition for the equivalence between approximation and interpolation.
Moreover, this result highlights the significance of the sets
$\bar{\mathcal{A}}_{\mathcal{F},T}$, which consist of families of functions that
can be arbitrarily approximated within a finite time interval. From the
perspective of approximation theory, a clear characterisation of the elements
within the sets $\bar{\mathcal{A}}_{\mathcal{F},T}$ would enable the
identification of functions that are easily approximated and yield results
regarding approximation rates. However, directly identifying the set
$\bar{\mathcal{A}}_{\mathcal{F},T}$ will be difficult, as it is an
infinite-dimensional problem and lacks a linear space structure. Nonetheless,
Proposition~\ref{prop:uni_time} provides a characterisation of the elements
within the set $\bar{\mathcal{A}}_{\mathcal{F},T}$ by utilizing the
interpolation times on finite sets of data points. This conversion from an
infinite-dimensional problem to a collection of finite-dimensional optimal
control problems may enable us to study $\bar{\mathcal{A}}_{\mathcal{F},T}$ using
well known methods in finite-dimensional optimal control theory, such as the
Pontryagin maximum principle~\cite{pontryagin2018mathematical}
and the Hamilton-Jacobi-Bellman equation~\cite{bellman1966dynamic}.
There are also possible connections with mean-field control~\cite{bongini2015.MeanFieldPontryaginMaximum,e2019MeanFieldOptimalControl,fornasier2014MeanFieldOptimalControl}
and $\Gamma$-convergence~\cite{fornasier2019MeanfieldOptimalControl}, which also study the connection between
infinite and a sequence of finite-dimensional problems.
These are promising future directions.

\section{Proof of Main Results}
In this section, we provide the proof of the main results in this paper, with detailed discussions. We begin with two preliminary results on the controllability of control systems. We say a control family $\mathcal F$ is \emph{symmetric}, if $f\in\mathcal F$ implies $-f\in\mathcal F$.

\subsection{Preliminaries}

To this end, we recall the famous Chow--Rashevsky theorem, which gives a sufficient condition of controllability of non-linear control system via Lie bracket generation.
For two smooth vector fields $f_1,f_2$ in $\mathbb R^d$, define the Lie bracket $[f_1, f_2] = \nabla_{x} f_2f_1 - \nabla_{x} f_1f_2$, and denote
$\operatorname{Lie}\mathcal F$ as the smallest subspace of smooth vector fields which is closed under Lie bracket operation and contains $\mathcal F$.
\begin{theorem}[Chow--Rashevsky theorem for $\mathbb R^d$]
    \label{thm:chow}
Consider a smooth domain $\Omega \subset \mathbb R^d$, if $\operatorname{Lie}\mathcal F|_q = \mathbb R^d$ for all $q \in \Omega$, then the system with control family $\mathcal F$ is controllable over $\Omega$. That is, for all $x, y \in \Omega$, there exists $\varphi\in\mathcal A_{\mathcal F}$ such that $\varphi(x)=y$.
\end{theorem}

The proof of Theorem~\ref{thm:chow} can be found in~\cite[Chap. 5]{agrachev2004control}. It provides a sufficient condition for controllability of a control system. Note that the vector fields considered in the theorem (as well as many related results) are assumed to be smooth, whereas the control family of continuous-time ResNets may not be. For the general Lipschitz control family (not necessarily smooth), the controllability still holds if we replace $\operatorname{Lie}\mathcal F$ with $\operatorname{span}\mathcal F$ in Theorem~\ref{thm:chow}. Specifically, we have the following result.

\begin{proposition}
    \label{prop:span}
        Consider a control system with symmetric control family $\mathcal F$, for a smooth domain $\Omega \subset \mathbb R^d$, if $\operatorname{span}\mathcal F|_q = \mathbb R^d$ for all $q \in \Omega$, then the system is controllable in $\Omega$.
\end{proposition}
\begin{proof}
    For $x\in\Omega$, denote
    $
        \mathcal O_{x}:=\{\varphi(x)\mid \varphi\in\mathcal A_{\mathcal F}\}
    $
    as the orbit through $x$.
    For any $q\in\mathcal O_x$, since $\operatorname{span} \mathcal F|_q=\mathbb R^d$, there exists $f_1,\cdots, f_d\in\mathcal F$ such that $f_1(q),\cdots,f_d(q)$ are linearly independent. Consider the map
    \begin{equation}
        \tau: \mathbb R^d\to\mathbb R^d, (t_1,\cdots,t_d)\to \varphi_{t_d}^{f_d}\circ \cdots\circ \varphi_{t_1}^{f_1}(q).
    \end{equation}
    Since $\mathcal F$ is symmetric, the map is well-defined.
     We have $\tau(0)=q$ and $\nabla \tau(0)=[f_1(q),\cdots,f_d(q)]$. Since $f_1(q),\cdots ,f_d(q)$ are linearly independent, the Jacobian matrix $\nabla \tau(0)$ is invertible. Therefore, $\tau$ is a local diffeomorphism at $0$. It follows that the image of $\tau$ contains a neighbourhood of $q$. Since $q$ is arbitrary and the image of $\tau$ is contained in $\mathcal O_x$, we have $\mathcal O_x$ is open for any $x$. Suppose there are $x,y$ such that $y\notin \mathcal{O}_x$, then $\Omega\setminus O_x\neq \emptyset$. However,
    $$\Omega\setminus O_x=\bigcup_{y\in \Omega\setminus O_x}(\mathcal{O}_y\cap \Omega)$$
    is a non-empty open set. We then have $\Omega=(\Omega\cap\mathcal{O}_x)\cup (\Omega\setminus \mathcal{O}_x)$
    is the union of two non-empty disjoint open set, which contradicts to the connectedness of $\Omega$. Therefore, there is only one orbit $\mathcal O$ such that $\mathcal O\supseteq\Omega$, i.e., the system is controllable on $\Omega$.
\end{proof}
\begin{remark}
    The condition for Lipschitz control family are more restrictive, since it always holds that $\operatorname{span}\mathcal F \subseteq \operatorname{Lie} \mathcal F$, provided the latter one is well-defined.
\end{remark}

\subsection{Proof of Theorem~\ref{thm:con_aff}}

In this subsection, we present the proof of Theorem~\ref{thm:con_aff}, the characterisation of UIP under affine invariance conditions. We also discuss some extension of the result in \Cref{subsubsec:extension}. The idea in proving Theorem~\ref{thm:con_aff} can be summarized as follows:

\begin{enumerate}
    \item  We first transform the UIP into a set of one-point control problems for high-dimensional systems \eqref{eq:Nd_system}.
    \item Then, we use Proposition~\ref{prop:span} to examine the controllability of each individual system to give a sufficient criterion for UIP.
    \item Finally, by using the affine invariance assumption and employing a Fourier transform technique, we find that the only exception for the criterion is the case of linear systems, thus giving a characterization of UIP.
\end{enumerate}

Recall that we assume $d \ge 2$ here.
For each positive integer $N$, we define
\begin{equation}
    \Omega_N=\big\{(x_1,\cdots,x_N)\mid \forall i\neq j ,x_i\neq x_j, i,j\in\{1,2\cdots, N\} \big\}.
\end{equation}
As we mentioned before, the simultaneous controllability of $N$ distinct data points is then equivalent to the classical controllability of the $Nd$ dimensional system
\begin{equation}
    \label{eq:Nd_system}
    \dot{X}_N(t)=\Big(f(x_1(t);{\theta}(t)), f(x_2(t);{\theta}(t)), \cdots , f(x_N(t);{\theta}(t))\Big),
\end{equation}
where $X_N(t)=(x_1(t),\cdots,x_N(t))\in\Omega_N\subset \mathbb R^{Nd}$.

The control family of the $(Nd)$-dimensional system~\eqref{eq:Nd_system} is
\begin{equation}
    \mathcal F_N:=\Big\{\big(f(x_1),f(x_2),\cdots,f(x_N)\big)\mid f\in \mathcal F\Big\}.
\end{equation}
As $\Omega_N$ is a connected open subset of $\mathbb R^{Nd}$ for all $d\ge 2$, we can directly apply Theorem~\ref{thm:chow} and Proposition~\ref{prop:span} to get a sufficient condition for the UIP:

A control system with control family $\mathcal F$ possesses UIP if
\begin{equation}
    \label{eq:controllability}
\operatorname{Span}\mathcal F_N(X_N)=\mathbb R^{Nd}\ \text{or}\ \operatorname{Lie}\mathcal F_N(X_N)=\mathbb R^{Nd}
\end{equation}
for any positive integer $N$ and $X_N\in\Omega_N\subset \mathbb R^{Nd}$. Note that the definition of Lie closure requires a higher regularity.

This condition serves as a sufficient condition for general control systems. However, by leveraging the affine invariance assumption on the control family $\mathcal F$, we are able to derive a characterization of UIP from the condition.
Now we present the proof of Theorem~\ref{thm:con_aff}.
\begin{proof}[Proof of Theorem~\ref{thm:con_aff}]

    Since the flow map of linear function are again linear, it suffices to prove the part of sufficiency.

    We prove it by contradiction. Suppose the UIP does not hold under the assumptions, then there must exist a positive integer $N$ and $X_N=(x_1,x_2,\cdots,x_N)\in\Omega_N$ such that $\operatorname{Span} \mathcal F_N(X_N)$ is contained in a hyperplane in $\mathbb R^{Nd}$. That is, there exist $c_{k,j}\in\mathbb R(1\le j\le d, 1\le k\le N)$, at least one being non-zero, such that for any $g=(g_1,\cdots, g_d)\in\mathcal F$,
    \begin{equation}
        \label{eq:rank_con_fail}
        \sum_{j = 1}^{d} \sum_{k = 1}^N c_{k,j} g_j(x_k) = c_{1,1}g_1(x_1)+\cdots+c_{1,d}g_d(x_1)+\cdots+c_{N,1}g_1(x_N)+\cdots+c_{N,d}g_d(x_N)=0.
    \end{equation}

    By assumption, there exists $f=(f_1,\cdots, f_d)\in \mathcal F$ with at least one $f_j$ being non-linear, without loss of generality we assume that $f_1:\mathbb R^d\to\mathbb R$ is non-linear. Since $\mathcal F$ is affine invariant, for any matrix $W, A\in\mathbb R^{d\times d}$ and vector $b\in\mathbb R^d$, it holds that $Wf(A\cdot-b)\in\mathcal F.$

    Choosing a non-zero $c_{m,l}$, we denote by $e_{l1}$ the matrix whose $(l,1)$-th entry is 1 and all other entries are zero. Let $g$ be $e_{l1}f(A\cdot-b)$ in~\eqref{eq:rank_con_fail}, we then have
    \begin{equation}
    \sum_{k=1}^N c_{k,l}f_1(Ax_k-b)=0
    \end{equation}
holds for all $A \in \mathbb R^{d\times d}$ and $b \in \mathbb R^{d}$.

Note that $f_1$ is of polynomial growth, and can be naturally regarded as a tempered distribution. Therefore, taking the Fourier transform with respect to $b$, we obtain
\begin{equation}
    \label{eq:tf-rank-fail}
    \left (\sum_{k = 1}^N c_{k,l} \exp(\mathbf i (\xi^\top Ax_k))\right )\hat f_1(\xi) = 0
\end{equation}
as a tempered distribution, for all matrix $A \in \mathbb R^d$ and $\xi\in\mathbb R^d$.

We now claim that \eqref{eq:tf-rank-fail} indicates $\supp \hat f_1$ is support at the origin. Suppose there exists $\xi_{0} \in \supp \hat f_1$ such that $\xi_{0}\neq 0$, then $A^\top\xi_{0}$ can take any value in $\mathbb R^d$ as $A$ goes through $\mathbb R^{d\times d}$. According to Lemma~\ref{lem:dist_supp}, this implies that for any $w\in\mathbb R^d$,
\begin{equation}
\sum_{k = 1}^N c_{k,l} \exp(\mathbf{i} (w^\top x_k))=0.
\end{equation}
As $x_1,\cdots, x_N$ are distinct, there exists $\tilde w\in\mathbb R^d$ such that $z_k:=\tilde w^\top x_k( k=1,\cdots, N)$ are distinct real numbers. Letting  $w = s\tilde w$ for $s \in \mathbb R$, we then have
\begin{equation}
v(s):=\sum_{k = 1}^N c_{k,l} \exp(\mathbf i sz_k)\equiv 0,
\end{equation}
for all $s\in\mathbb R$. It is straightforward to see that all $c_{k,l}$ are zero. In fact, the $j$-th($0\le j\le N-1$) derivative of $v$ at $0$ gives:
\begin{equation}
    v^{(j)}(0)=\sum_{k = 1}^N c_{k,l}(\mathbf iz_k)^j=0.
\end{equation}

Since all the $z_k$ are distinct, the Vandermonde matrix $[(\mathbf iz_k)^{j-1}]_{1\le k,j\le N}$ is invertible. It then indicates that all $c_{k,l}, k=1,\cdots,N$ must be zero, which contradicts to our assumption that $c_{m,l}\neq 0$ for some $m$. Therefore, $\supp \hat f_1$ must be $\{0\}$. According to Theorem 5 in~\cite[Appendix C]{lax2002FunctionalAnalysis}, this indicates that $f_1$ is a polynomial function. However, this contradicts to our assumption that $f_1$ is non-linear and globally Lipschitz. Hence, the UIP holds for systems with non-linear affine invariant control family $\mathcal F$.
\end{proof}

The following lemma on Fourier transform is required in our proof of Theorem~\ref{thm:con_aff}. In fact, it is an exercise of~\cite{lax2002FunctionalAnalysis}.

\begin{lemma}
    \label{lem:dist_supp}
Let $h:\mathbb R^d\to\mathbb R$ is a smooth function whose partial derivatives in each order are all of polynomial growth. Suppose for some $f\in\mathcal S^\prime(\mathbb R^d)$, $h\cdot f=0$ as a tempered distribution. Then, for any $x_0\in\mathbb R^d$ such that $h(x_0)\neq 0$, $x_0\notin \supp f$.
\end{lemma}
\begin{proof}
Since $hf=0$ as a tempered distribution, we have $f(h\cdot g)=0$ for all test function $g\in\mathcal S(\mathbb R^d)$. Because $h(x_0)\neq 0$, there exists a neighbourhood $U$ of $x_0$ such that $h\neq 0$ on $\bar U$. The smooth function $\frac{1}{h(x)}$ can be then defined on $\bar U$. Now we choose a smooth bump function $B(x)$ defined on $\mathbb R^d$ such that $\supp B\subset U$ and $B(x)\equiv 1$ on a neighborhood $V\subseteq U$ of $x_0$. Therefore, the function
\begin{equation}
    l(x):=\left\{
    \begin{alignedat}{2}
         &B(x)h(x),  \ && x\in U,\\
        &0, &&x\in \mathbb R^d\setminus U
    \end{alignedat}
    \right.
\end{equation}
is in $\mathcal S(\mathbb R^d)$ and coincide with $\frac{1}{h(x)}$ on $V$. Consequently, for any $g\in\mathcal S(\mathbb R^d)$ with $\supp g\subseteq V$, we have
$
    f(g)=f(h\cdot (l\cdot g))=0
$
Therefore, $\supp f\subset \mathbb R^d\setminus V$, which implies $x_0\notin \supp f$.
\end{proof}

\subsection{Discussions and Extensions}
In this subsection, we present some discussions and extensions of Theorem~\ref{thm:con_aff}. We will first discuss the concept of UIP and the result of Theorem~\ref{thm:chow} for one-dimensional systems. After that, we extend the approach in the proof of Theorem~\ref{thm:con_aff} to study UIP under weaker affine invariance assumptions. At the end, we will discuss the concept of attainable set for locally Lipschitz control families.
\label{subsubsec:extension}
\subsubsection{One-dimensional systems}
\label{subsubsec:1dsys}

    In the previous discussions, we only considered the concept of UIP and UAP for dimensions more than two. The main difference for the one-dimensional system is that the flow maps of one-dimensional dynamical systems are always increasing. As a consequence, the attainable set $\mathcal A_{\mathcal F}$ consists exclusively of increasing functions, limiting the expressive power of $\mathcal A_{\mathcal F}$. Therefore, for the UIP and UAP of one-dimensional systems, we additionally require
    \begin{enumerate}
        \item  $F$ to be increasing in Definition~\ref{def:uap}, and
        \item the data $\{(x_i,y_i)\}_{i=1}^N$ to satisfy that $x_1<\cdots <x_N$ and $y_1<\cdots< y_N$ in Definition~\ref{def:uip}.
    \end{enumerate}

    Note that the UIP of a one-dimensional system is equivalent to the controllability of system~\eqref{eq:Nd_system} over the domain
    $
     \{(x_1,\cdots, x_N) \in \mathbb R^{N}\mid x_1<\cdots<x_N\}.
    $
It follows that a one-dimensional control system possesses the UIP if and only if there exists $f\in \mathcal F$ which is non-linear, and the proof is similar to that in Theorem~\ref{thm:con_aff}.

    A crucial observation for one-dimensional control systems is that the UIP naturally implies the UAP. Since for any strictly increasing continuous function $h:\mathbb R\to\mathbb R$, interval $[a,b]$ and $\varepsilon>0$, there exists a sequence $a= x_1<x_2<\cdots< x_{N}=b$ such that $h(x_i)-h(x_{i-1})<\varepsilon$ for all $i=2,\cdots N$. The UIP implies that there exists $\varphi\in\mathcal A_{\mathcal F}$ such that $\varphi(x_i)=h(x_i)$ for all $i$. Therefore,
    \begin{equation}
        -\varepsilon<h(x_{i-1})-\varphi(x_i)\le h(x)-\varphi(x)\le h(x_i)-\varphi(x_{i-1})<\varepsilon
    \end{equation}
    whenever $x\in[x_{i-1},x_{i}]$. This indicates that $|h-\varphi|_{C(K)}<\varepsilon$. Given that the set of strictly increasing continuous functions is dense in the set of increasing continuous functions with respect to the $C(K)$ norm, we can deduce the UAP of systems with UIP.

\subsubsection{Conditions under weaker assumptions}
\label{subsubsec:weaker_aff}
    We weaken the assumptions on the affine invariance. Specifically, we have the following result.
    \begin{proposition}
        \label{prop:con_weaker_aff}
        Suppose a control family $\mathcal F$ satisfies that for any $f\in\mathcal F$, $Df(A\cdot-b)\in\mathcal F$ for any diagonal matrix $D$ and $A\in \mathscr A\subseteq \mathbb R^{d\times d}$, then $\mathcal A_{\mathcal F}$ possesses the UIP if for each coordinate index $j\in\{1,\cdots, d\}$, the set
    \begin{equation}
        \Xi_j(\mathscr A) = \{A^\top\xi\mid A\in\mathscr{A}, \xi\in \supp \hat f_j, \text{where}\ f_j \ \text{is the}\ j\text{-th coordinate of some}\ f\in\mathcal F \}
    \end{equation}
    contains an open set in $\mathbb R^d$

    \end{proposition}
    \begin{proof}
        Similar to the proof of Theorem~\ref{thm:con_aff}, the UIP fails only if there exists distinct $x_1,\cdots, x_N$ and $c_{k,j}\in\mathbb R$ that are not all zeros such that for any $f=(f_1,\cdots, f_d)\in\mathcal F$,
        \begin{equation}
             \sum_{j = 1}^d \sum_{k= 1}^{N}c_{k,j} f_j(x_{k}) =0.
        \end{equation}
Now we fix an index $J \in \{1,2,\cdots, d\}.$ Since $(f_1,\cdots, f_d) \in \mathcal F$, by assumption, it holds that $(0, \cdots, 0, f_{J}(A\cdot -b), 0,\cdots, 0) \in \mathcal F.$
Therefore,
\begin{equation}
    \sum_{k = 1}^{N}c_{k,J} f_J(Ax_{k}-b) =0
\end{equation}
for all $x_i$, $A\in\mathscr A$ and $b\in\mathbb R^d$. Taking Fourier transform with respect to $b$ gives that
\begin{equation}
    \left (\sum_{k = 1}^N c_{k,J} \exp(\mathbf i (\xi^\top Ax_k))\right )\hat f_k(\xi) = 0
\end{equation}
Therefore, according to Lemma~\ref{lem:dist_supp}, for each $\xi\in\supp \hat g_k$ and $A\in\mathscr A$.
\begin{equation}
    \sum_{k = 1}^N c_{k,J} \exp(\mathbf i (\xi^\top Ax_k))=0.
\end{equation}
Based on the assumption, $A^\top \xi$ can take all the values in some open set $U$ in $\mathbb R^d$ when $f$ goes through $\mathcal F$, $\xi$ goes through $\supp f_j$ and $A$ goes through $\mathscr A$. Consequently, we can take some $A,\xi$ with $A^\top \xi=w$ such that $w^\top x_1,\cdots, w^\top x_N$ are $N$ distinct real numbers. Since $U$ is open, we know that there is some $\delta>0$ such that $sw\in U$ for all $s\in(1-\delta, 1+\delta)$. Therefore, we have
\begin{equation}
    \sum_{k = 1}^N c_{k,J} \exp(\mathbf i s( aw^\top x_k))=0,
\end{equation}
for all $s\in (1-\delta, 1+\delta)$. Because the left-hand side in the above equality is an analytic function in $s$, this implies that the equality holds for all $s\in\mathbb R$. Consequently, using the same discussion as in the proof of Theorem~\ref{thm:con_aff}, we deduce that all $c_{k,l}$ are all zero, which leads to a contradiction.
\end{proof}

Although the condition in Proposition~\ref{prop:con_weaker_aff} may seem technical, it  can be applied to specific cases to get relatively concise corollaries. The following corollaries can be deduced by directly checking the conditions in Proposition~\ref{prop:con_weaker_aff}.

\begin{corollary}
    Consider
    \begin{equation*}
    \mathcal F:=\{W\sigma(A\cdot +b)\mid W, A\in \mathbb R^{d\times d},b\in\mathbb R^d, W\ \text{is diagonal}, A\in\mathscr A\subseteq\mathbb R^{d\times d}\},
    \end{equation*}
    where $\sigma:\mathbb R\to\mathbb R$ is a non-linear Lipschitz function that is applied element-wise to a vector in $\mathbb R^d$. Then, the UIP holds if for any $l=1,\cdots, d$, the set
    $\{v\in \mathbb R^d\mid v \ \text{is the l-th row of}\ A\ \text{for some}\ A\in\mathscr A \}$
    equals to $\mathbb R^d$. In particular, $\mathscr A$ can be the set of all cyclic matrices or the set of rank-1 matrices. This implies that continuous ResNet with many sharing weights in the transformation matrix can still possess the UIP.
\end{corollary}
\begin{corollary}
    Suppose there exists $(f_1,\cdots, f_d)\in\mathcal F$ such that $\supp \hat f_j$ contains an open set in $\mathbb R^d$ for some $j$, then the UIP hold if for any $f\in\mathcal F$, $W$ and $b$, $Wf(\cdot-b)\in\mathcal F$. That means, we do not need the scaling transform.
\end{corollary}
\begin{corollary}
    Suppose for each $j=1,\cdots, d$, there exists some $(f_1,\cdots, f_d)\in\mathcal F$ with some $(\xi_1,\cdots, \xi_d)\in \hat f_j$ such that $\xi_1\cdots\xi_d\neq 0$. In this case, UIP holds if $f\in\mathcal F$ implies $Df(A\cdot-b)\in\mathcal F$ for all diagonal matrices $D$ and $A$.
\end{corollary}

\subsubsection{Locally Lipschitz control family}
\label{subsubsec:loc_lip}
 So far, we assume that a control family consisting of globally Lipschitz functions, whereas the polynomial function class is excluded. Let us now consider some non-Lipschitz case. The difficulty in studying such control systems is that the flow map is not globally well-defined in general.
\begin{example}
    Take $\dot{x} = x^2$ as an example, for initial data $x(0) = a > 0$, the solution $x(t) = \frac{1}{1/a - t}$, which will blow up in time $\frac{1}{a}$. Consequently, the global flow map might not exist for any time $t > 0$.

\end{example}

A similar result holds for general non-Lipschitz $f$, there may not be a uniform time $T>0$ such that the solution of ~\eqref{eq:con_sys} exists in $[0,T]$ for all initial values $x_0$. However, by sacrificing global well-posedness,
    the family of flow maps for locally Lipschitz control family can be defined correspondingly, with assuming all the local flow maps are well-defined.

    For a given bounded set $K$, we can only consider $x_0$ in a bounded set $K$ and only consider those controls ${\theta}(\cdot):[0,T]\to \Theta$ such that the solution of~\eqref{eq:con_sys} exists in $[0,T]$ for all $x_0\in K$. In the above example, the local flow map is well-defined whenever $T < \inf_{x \in K} \frac{1}{x}.$ For each $f \in \mathcal F$ and a compact set $K$, there exists $T_{\max}(f, K)>0$ such that the local flow map exists for all $t \in (0, T_{\max}(f,K))$.   Define
    \begin{equation}
        \Phi(\mathcal{F}, T):=\left\{\varphi_{t_k}^{f_k}\circ \varphi_{t_{k-1}}^{f_{k-1}}\circ \cdots\circ \varphi_{t_1}^{f_1}(\cdot) \mid t_1+\cdots t_k=T,f_1,\cdots,f_k\in \mathcal F , k\ge 1\right\},
    \end{equation}
    where the constraint of the time $t_i$ are specified as follows:
    \begin{enumerate}
    \item First, $0 < t_1 < T_{\max}(f_1, K)$. Define $K_1 = \varphi_{t_1}^{f_1}(K).$
    \item Suppose that the condition on $t_j$ and $K_j$ are given. Let $0 < t_{j+1} < T_{\max}(f_{j+1}, K_{j})$, and define $K_{j+1} = \varphi_{t_{j+1}}^{f_{j+1}}(K_j)$.
    \end{enumerate}
    Note that in some cases, the local attainable set might be empty for some time horizon $T$. Finally, we can define $\mathcal A_{\mathcal F,T}(K)$ and $\mathcal A_{\mathcal F}(K)$ for a local Lipschitz function class $\mathcal F$, notice that the definition of both the function class depend on $K$.

    Therefore, the concept of UIP and UAP can be extended to the system with locally Lipschitz control families (i.e., consisting of locally Lipschitz function), as long as we replace $\mathcal A_{\mathcal F}$ with $\mathcal A_{\mathcal F}(K)$ in Definition~\ref{def:uap} and Definition~\ref{def:uip}. In the following subsection, we will extend to consider the UAP and UIP for locally Lipschitz control families. In this case, the controllability condition in Chow-Rashavsky theorem(Theorem~\ref{thm:chow}) and Proposition~\ref{prop:span} also holds, since they only concern the flow maps locally. Therefore, we can still study UIP for systems with locally Lipschitz control family based on these controllability conditions.

\subsection{Relationship between UAP and UIP}
This subsection is devoted to the relationship between UIP and UAP.
First, we recall the result in~\cite{li2022DeepLearningDynamical} that deduces UAP from UIP. Based on that result, we can further derive UAP from the sufficient conditions for UIP established in the previous sections.

The main body of this subsection illustrates that UAP and UIP cannot be derived from each other in the context of general control systems, by verifying two specific examples ~\eqref{eq:uip_not_uap} and \eqref{eq:uap_not_uip}. This actually proves Theorem~\ref{thm:uap_uip}. Subsequently, we provide the proof of Proposition~\ref{prop:uni_time}, bridging the relation between universal interpolation and approximation.
\subsubsection{A sufficient condition for UIP to imply UAP}
In many previous works\cite{li2022DeepLearningDynamical,ruiz-balet2023NeuralODEControl,tabuada2020UniversalApproximationPower}, the UAP of continuous-time ResNets has been derived from the established UIP using a similar technique. Specifically, these works consider a specific type of ``shrinking maps'' from $\mathbb{R}^d$ to $\mathbb{R}^d$ defined on a given compact set $K$. These maps are continuous and uniformly bounded, capable of contracting most parts of $K$ (in terms of measure) into a finite set of points within $K$. Thus, if the attainable set $\mathcal{A}_{\mathcal{F}}$ can approximate such shrinking maps in the $L^{\infty}$ sense, we can first contract the domain $K$ into a finite set of data points and then apply interpolation results to this finite set, thereby establishing the UAP for the control system. The following proposition provides a specific condition based on this methodology. For a detailed proof, refer to the proof of Theorem 2.1 in Section 4.3 of~\cite{li2022DeepLearningDynamical}.
\begin{proposition}
    \label{prop:uap_suf}
Let $\mathcal F$ be the control family of a control system that possesses the UIP. If for any compact set $K$ and increasing function $h:\mathbb R\to\mathbb R$, the attainable set $\mathcal A_{\mathcal F}$ is capable of approximating the function $(h(x_1),\cdots, h(x_d))$ in the $L^\infty(K)$ norm, then the system also possesses the UAP in $L^p$ sense for all $1\le p<\infty$.
\end{proposition}
Combined with Theorem~\ref{thm:con_aff}, and characterisation of UIP in 1 dimension, a corollary of Proposition~\ref{prop:uap_suf} is the UAP of continuous-time ResNet with non-linear activation function applied element-wisely.
\begin{corollary}
    \label{cor:resnet_uap}
    Any continuous-time ResNet of the form $\dot x(t)=W(t)\sigma(A(t)x(t)+b(t))$, where $W(t),A(t)\in\mathbb R^{d\times d}$, $b(t)\in\mathbb R^d$ and $\sigma:\mathbb R\to\mathbb R$ is a non-linear Lipschitz activation function applied element-wise to a vector in $\mathbb R^d$, possesses the UAP.
\end{corollary}
\begin{proof}
The UIP of such system is directly followed from Theorem~\ref{thm:con_aff}. As $\sigma$ is non-linear, the 1-dimensional system $\dot x(t)=\sigma(a(t)x(t)+c(t))$ possesses the UAP, see \Cref{subsubsec:1dsys}. Since for each $j=1,\cdots, d$, the function $(0,\cdots, 0, \sigma(x_j),0,\cdots,0)$(the $j$-th coordinate is $\sigma(x_j)$ and all others are zero) is in the control family, by composition, this implies that the condition in Proposition~\ref{prop:uap_suf} holds. Consequently, by Proposition~\ref{prop:uap_suf}, the UAP holds.
\end{proof}
Therefore, our result on UIP of certain ResNet structure can be extended to the UAP. However, the condition in Proposition~\ref{prop:uap_suf} just serves as a sufficient one and may not work for general control systems.

\subsubsection{UIP does not Imply UAP: An Example}
\label{subsubsec:uip_not_uap}
The examples presented below use the fact that divergence-free vector fields generates volume-preserving flows.
Therefore, such a control systems cannot approximate the target mapping with varying ``volume''.
On the other hand, we can demonstrate that such systems can still possess the UIP, as $\operatorname{Lie}\mathcal{F}$ can potentially be a large set, satisfying the rank condition \eqref{eq:controllability}. Specifically, we provide an example in two dimensions using polynomial vector fields.

For any two-dimensional vector field $V=(v_1(x_1,x_2),v_2(x_1,x_2)):\mathbb R^2\to \mathbb R^2$, if $\partial_{x_1}v_1+\partial_{x_2}v_2\equiv 0$, then the flow maps of $F$ are area-preserving. That is, $m(\varphi_t^V(U))=m(U)$ for any time $t$ and measurable set $U\subset \mathbb R^2$.

For any smooth function $f:\mathbb R^2\to \mathbb R$, let $ v(f)$ be the vector field $\curl f = (-f_{x_2},f_{x_1})$. Then we have $\operatorname{div} v(f)=0$. For smooth functions $f$ and $g$, a direct computation gives:
\begin{equation}
\label{eq:divlie}
\begin{aligned}
    &[v(f), v(g)]=[(-f_{x_2},f_{x_1}),(-g_{x_2},g_{x_1})]\\
    &=(f_{x_2}g_{x_1x_2}-g_{x_2}f_{x_1x_2}-f_{x_1}g_{x_2x_2}+g_{x_1}f_{x_2x_2},-f_{x_2}g_{x_1x_1}+g_{x_2}f_{x_1x_1}+f_{x_1}g_{x_1x_2}-g_{x_1}f_{x_1x_2})\\
    &=(-\partial_{x_2}(f_{x_1}g_{x_2}-f_{x_2}g_{x_1}), \partial_{x_1}(f_{x_1}g_{x_2}-f_{x_2}g_{x_1}))\\
    &=v(f_{x_1}g_{x_2}-f_{x_2}g_{x_1})
\end{aligned}
\end{equation}
Using~\eqref{eq:divlie}, we can find a small family $\mathcal{F}$ of two-dimensional polynomial vector fields such that $\operatorname{Lie}\mathcal F$ is the set of all divergence-free polynomial vector fields in $\mathbb R^2$.
\begin{lemma}
\label{lem:vollie}
    Let $\mathcal F:=\{ v(x_1),v(x_2),v(x_1^2x_2^2)\}=\{(0,1),(-1,0),(-2x_1^2x_2,2x_1x_2^2)\}$. Then, $\operatorname{Lie}\mathcal F=v(\mathbb R[x_1,x_2]):=\{v(f)\mid f\in \mathbb R[x_1,x_2]\}$.
\end{lemma}
\begin{proof}
From~\eqref{eq:divlie}, we have $\frac{1}{2}[v(x_1),v(x_1^2x_2^2)]=v(x_1^2x_2)\in \operatorname{Lie} \mathcal F$. From the relation
$$\frac{1}{2}[v(x_1^{m+1}x_2^{m}),v(x_1^2x_2^2)]=v(x_1^{m+2}x_2^{m+1}),$$
we know by induction that  $v(x_1^{m+1}x_2^{m})\in \operatorname{Lie} \mathcal F$ for all $m\in \mathbb N^*$.

For any $v(x_1^mx_2^n)\in \operatorname{Lie} \mathcal F$ with $m,n\in \mathbb{N}^*$, we have  $\frac{1}{n}[v(x_1),v(x_1^mx_2^n)]=v(x_1^mx_2^{n-1})$ and $\frac{1}{m}[v(x_1^mx_2^n),v(x_2)]=v(x_1^{m-1}x_2^{n})$ are in $\operatorname{Lie}\mathcal F$. Thus, for any $a,b\in \mathbb N$, from $x_1^{a+b+1}x_2^{a+b}\in \operatorname{Lie} \mathcal F$, we can show by induction that  $x_1^ax_2^b\in \operatorname{Lie} \mathcal F$.

Since $v(x_1^ax_2^b)\in \operatorname{Lie} \mathcal F$ for any monomial $x_1^ax_2^b$, we have $\operatorname{Lie}\mathcal F=\{v(f)\mid f\in \mathbb R[x_1,x_2]\}$.
\end{proof}
\begin{proposition}
    \label{prop:uip_not_uap}
    The two-dimension control system
    \begin{equation}
        \label{eq:uip_not_uap}
        \left\{
        \begin{aligned}
        &\dot x_1=-\theta_1-2\theta_3 x_1^2x_2,\\
        &\dot x_2=\theta_2+2\theta_3 x_1x_2^2
        \end{aligned}
        \right.
    \end{equation}
with control $\bm \theta=(\theta_1,\theta_2,\theta_3)\in \mathbb R^3$ possesses the UIP, but does not possess the UAP.
\end{proposition}
\begin{proof}

   This system is just $(\dot x_1,\dot x_2)=\theta_1 v(x_1)+\theta_2 v(x_2)+\theta_3 v(x_1^2x_2^2)$.

   By Lemma~\ref{lem:vollie}, the control family $\mathcal F=\{\theta_1 v(x_1)+\theta_2 v(x_2)+\theta_3 v(x_1^2x_2^2)\mid \theta_1,\theta_2,\theta_3\in \mathbb R\}$ satisfies $\operatorname{Lie}\mathcal F=v(\mathbb R[x_1,x_2])$.  To show the system has UIP, by~\eqref{eq:controllability}, we only need to verify that for any given distinct samples $\{(x_{1,i},x_{2,i})\}_{i=1}^{N}\subseteq \mathbb R^2$, there exist polynomials $f_1,\cdots,f_{2N} \in \mathbb R[x_1,x_2]$ such that  the vectors
   \begin{equation}
    \{F_i:=(v(f_i)(x_{1,1},x_{2,1}),\cdots, v(f_i)(x_{1,N},x_{2,N}))\}_{i=1}^{2N},
   \end{equation}
are linearly independent.

   Since the samples are distinct, there exists $a, b\in \mathbb R, a\neq b$ such that $\{ax_{1,i}-x_{2,i}\}_{i=1}^N$ and $\{bx_{1,i}-x_{2,i}\}_{i=1}^N$ are both sets of $N$ distinct real numbers. We choose the $2N$ polynomials as $f_i(x_1,x_2)=(ax_1-x_2)^i, i=1,2\cdots, N, f_j(x_1,x_2)=(bx_1-x_2)^{j-N}, j=N+1,\cdots, 2N$. For convenience, we make the permutation of coordinates
   \begin{equation}
       (z_1, z_2,\cdots, z_{2N})\mapsto (z_1, z_3,\cdots, z_{2N-1}, z_2,z_4,\cdots,z_{2N})
   \end{equation}
   on $F_i$ and consider
   $$
   \tilde F_i:=(-\partial_{x_2}f_i(x_{1,1},x_{2,1}),\cdots,-\partial_{x_2}f_i(x_{1,N},x_{2,N}), \partial_{x_1}f_i(x_{1,1},x_{2,1})\cdots, \partial_{x_1}f_i(x_{1,N},x_{2,N})),
   $$
for $i=1,\cdots, 2N$. By a direct computation, the matrix $(\tilde F_1,\cdots,\tilde F_{2N})^\top$ has the form:
\begin{equation}
(\tilde F_1,\cdots,\tilde F_{2N})^\top=
   \left ( \begin{array}{cc}
      A   &  B\\
      aA   & bB
    \end{array}
    \right ),
\end{equation}
where $A$ is the $n\times n$ Vandermonde matrix of $ax_{1,1}-x_{2,1},\cdots, ax_{1,N}-x_{2,N} $, and $B$ is the Vandermonde matrix of $bx_{1,1}-x_{2,1},\cdots, bx_{1,N}-x_{2,N} $. Thus, $A$ and $B$ are both invertible. After a simple row transformation, $(\tilde F_1,\cdots,\tilde F_{2N})^\top$ is equivalent to
\begin{equation}
   \left ( \begin{array}{cc}
      A   &  B\\
      0   & (b-a)B
    \end{array}
    \right ).
\end{equation}
Since $b\neq a$, the matrix $(\tilde F_1,\cdots,\tilde F_{2N})^\top$ is invertible. As $\tilde F_i$ is just a coordinate permutation of $F_i$($i=1,\cdots,2N$), it follows that the vectors $\{F_i\}_{i=1}^{2N}$ are linearly independent.

Now we show that the above control system does not possess the UAP. Consider the compact set $K$ as the closed unit disc $B_1(0)$ in $\mathbb R^2$. We know that for any flow map $\varphi\in\mathcal A_{\mathcal F}$, the area of $\varphi(K)$ is invariant as $\pi$. Now we consider the constant function $F\equiv (0,0)$ as the target. It follows that
$$
\Vert F-\varphi_t\Vert_{L^2(K)}^2=\int_K \Vert\varphi_t(x_1,x_2)\Vert^2 \,dx\,dy=\int_{\varphi(K)} (u^2+v^2) \,du\,dv\ge 2\pi \int_{0}^1 r^3 \,dr=\frac \pi 2,
$$
Therefore, the system cannot approximate $F$.

\end{proof}
\subsubsection{UAP does not Imply UIP: An Example}
\label{subsubsec:uap_not_uip}
As the UAP we defined is in $L^p$ sense, it is still possible for a control system to possess UAP if it cannot interpolate data points over a measure zero set in $\mathbb R^d$. The following proposition gives a specific example of a 2-dimensional system with UAP but has a fixed point at the origin.
\begin{proposition}
\label{prop:uap_not_uip}
    The 2-dimensional control system
    \begin{equation}
        \label{eq:uap_not_uip}
        \left\{
        \begin{aligned}
        &\dot x_1=\theta_1x_1^3+\theta_2x_1^2+\theta_3x_2,\\
        &\dot x_2=\theta_4x_2^3+\theta_5x_2^2+\theta_6x_1
        \end{aligned}
        \right.
    \end{equation}
    possesses UAP but does not possess UIP.
\end{proposition}

It is obvious that all the flow maps of system~\eqref{eq:uap_not_uip} have a fixed point at the origin. Therefore, we only need to show that system~\eqref{eq:uap_not_uip} possesses the UAP. The key idea is to show that $\mathcal A_{\mathcal F}$ can still interpolate any finite set of data points without input value and output value at the origin, and then apply the sufficient condition used in Proposition~\ref{prop:uap_suf} to derive UAP from the weakened UIP. The proof is almost the same with the proof of Theorem 4.11 in Section 4.3 of~\cite{li2022DeepLearningDynamical}, provided the following two auxiliary results.
\begin{lemma}
    \label{lem:1}
    For any compact set $K\subset \mathbb R$, increasing function $h:\mathbb R\to\mathbb R$ with $h(0)=0$ and $\varepsilon >0$, there exists a flow map $\varphi$ of the following 1-dimensional control system
    \begin{equation}
        \label{eq:lem_1d_sys}
        \dot x=\theta_1 x^3+\theta_2 x^2
    \end{equation}
    with control $\bm \theta=(\theta_1,\theta_2)$, such that $\Vert\varphi-h\Vert_{C(K)}<\varepsilon$
\end{lemma}
\begin{proof}
    It is straightforward to see that $\varphi(0)=0$ for all $\varphi\in\mathcal A_{\mathcal F}$. We first show that the set of flow maps $\mathcal A_{\mathcal F}$ can interpolate any dataset $\{(x_i,y_i)\}_{i=1}^{M+N}$, where $M,N$ are positive integers, $x_1<\cdots< x_M<0<x_{M+1}<\cdots <x_{M+N}$ and $y_1<\cdots< y_M<0<y_{M+1}<\cdots< y_{M+N}$. Similar to the discussion in Section~\ref{subsubsec:1dsys}, it only requires to check the controllability of the $(M+N)$-fold system on the following domain:
\begin{equation}
    \Omega_{M,N}:=\left \{(x_1,\cdots,x_{M+N})\mid x_1<\cdots< x_M<0<x_{M+1}<\cdots <x_{M+N}\right \}\subset \mathbb R^{M+N}.
\end{equation}
A basis of the control family $\mathcal F$ of system~\eqref{eq:lem_1d_sys} is $x^3$ and $x^2$. Apply the equality $[x^n,x^2]=(n-2)x^{n+1}$ inductively for all $n\ge 3$, we can show that $x^n\in \operatorname{Lie}\mathcal F$ for all $n\ge 3$. Therefore, for any $X=(x_1,\cdots,x_{M+N})\in\Omega_{M,N}$, the vectors $(x_1^n,\cdots, x_{M+N}^n)\in \operatorname{Lie} \mathcal F^{M+N}(X)$ for all $n$. Since $x_1,\cdots, x_{M+N}$ are all distinct and non-zero in $\Omega_{M+N}$, $\{(x_1^n,\cdots, x_{M+N}^n)\}_{n=2}^{M+N+1}$ form a basis of $\mathbb R^{M+N}$. By Theorem~\ref{thm:chow}, the $(M+N)$-fold system is controllable for all $M$ and $N$. Therefore, for any strictly increasing continuous function $h$ with $h(0)=0$, $\mathcal A_{\mathcal F}$ can interpolate any finite data points sampled from the graph of $h$.

Consequently, similar to we have discussed in Section~\ref{subsubsec:1dsys}, that implies that $\mathcal A_{\mathcal F}$ can approximate any increasing function $h$ with $h(0)=0$. Given that the set of strictly increasing continuous functions is dense in the set of increasing continuous functions with respect to the $C(K)$ norm, we deduce the conclusion of the lemma.
\end{proof}

\begin{lemma}
    \label{lem:2}
    System~\eqref{eq:uap_not_uip} can interpolate any finite set of data points $\{(x_i,y_i)\}_{i=1}^{N}$ such that $\{x_i\}_{i=1}^{N}$ and $\{y_i\}_{i=1}^{N}$ are both $N$ distinct and \emph{non-zero} points in $\mathbb R^2$.
\end{lemma}
\begin{proof}
    We need to check the controllability of the $N$-fold system on the following domain:
    \begin{equation}
        \tilde\Omega_N:=\{(x_1,\cdots,x_N)\mid 0\neq x_i\in\mathbb R^2.\ \forall i\neq j, x_i\neq x_j. i,j\in\{1,\cdots, N\} \}\subset \mathbb R^{2N}
    \end{equation}
The control family $\mathcal F$ of system~\eqref{eq:uap_not_uip} is spanned by the vectors
\begin{equation}
(x_1^2,0), (x_1^3,0), (0,x_2^2),(0,x_2^3),(x_2,0),(0,x_1).
\end{equation}
Direct calculations show that $(x_1^n,0),(0,x_2^n)\in\operatorname{Lie} \mathcal F$ for all $n\ge 2$. Direct computation gives
\begin{equation}
    [(x_1^n,0),(0,x_1)]=(0,x_1^n),[(0,x_2^n),(x_2,0)]=(x_2^n,0),
\end{equation}
\begin{equation}
    [(x_1^n,0),(x_2^m,0)]=(-nx_1^{n-1}x_2^m,0),\quad[(0,x_1^n),(0,x_2^m)]=(0,mx_1^nx_2^{m-1}).
\end{equation}
A combination of these identities indicates that all the vector fields $(x^iy^j,0)$ and $(0,x^iy^j)$ with $i+j\ge 2, i,j\ge 0$ are in $\operatorname{Lie} \mathcal F$. This implies that, for any polynomial $p_1(x_1,x_2)$ and $p_2(x_1,x_2)$ with vanishing linear and constant terms, the vector $(p_1,p_2)\in\operatorname{Lie}\mathcal F$. Therefore, by the interpolation property of polynomials, we know that $\operatorname{Lie}\mathcal F(x)=\mathbb R^{2N}$ for all $X\in\tilde\Omega_N$. This completes the proof.
\end{proof}
\begin{proof}[Proof of Proposition~\ref{prop:uap_not_uip}]
    Lemma~\ref{lem:1} implies that the set of flow maps $\mathcal A_{\mathcal F}$ of system~\eqref{eq:uap_not_uip} can approximately interpolate any finite set of data points $\{(x_i,y_i)\}_{i=1}^N$ with $x_i\neq 0$.

    When taking the control $\theta_2,\theta_6=0$, Lemma~\ref{lem:2} shows that $\mathcal A_{\mathcal F}$ can approximate the function $(h(x_1),h(x_2))$ in $C([-L,L]^2)$ for any $h:\mathbb R\to\mathbb R$ that is increasing and satisfies $h(0)=0$. Specifically, $\mathcal A_{\mathcal F}$ can approximate a family of ``shrinking functions'', similar to the one defined in Section 4.3 of~\cite{li2022DeepLearningDynamical}, over any cube $[-L,L]\times [-L,L]$ for $L>0$.

    Therefore, following essentially the same approach as the proof of Theorem 4.11 in~\cite{li2022DeepLearningDynamical}, we show that $\mathcal A_{\mathcal F}$ can approximate any continuous function $F$ with $F(0,0)=(0,0)$ in $L^p$ sense. Since any continuous function can be approximated by continuous functions vanishes at $0$ in $L^p(K)$, it follows that system~\eqref{eq:uap_not_uip} can approximate any continuous function in $L^p(K)$ sense.

\end{proof}
\subsubsection{Proof of Proposition~\ref{prop:uni_time}}

Finally, we give the proof of Proposition~\ref{prop:uni_time}. The key idea is that the uniformity assumptions in Proposition~\ref{prop:uni_time} actually implies the precompactness of $\mathcal A_{\mathcal F,T}$ in $C(K)$. This property establishes the equivalence between approximation and universal interpolation of a target function using $\mathcal A_{\mathcal F}$.
The following result can be directly derived from the classical ODE argument.
\begin{lemma}
    $\mathcal A_{\mathcal F,T}$ is uniformly Lipschitz and uniformly bounded.
\end{lemma}
\begin{proof}
    For any $f\in\mathcal F$ and the dynamical system
    \begin{equation}
        \label{eq:ode}
    \dot x(t)=f(x(t)),
    \end{equation}
we have
\begin{equation}
    |x(t)|_1=|x(0)+\int_0^t f(x(s))~\mathrm ds|_1\le |x(0)|_1+\int_0^t (c_1+c_2|x(s)|)~\mathrm ds.
\end{equation}
    Applying Gr\"onwall's inequality gives
    \begin{equation}
        |\varphi_t^f(x(0))|_1=|x(t)|_1\le (|x(0)|_1+c_1t)e^{c_2t},
    \end{equation}
    for all $t\ge 0$ and $x(0)\in\mathbb R^d$.

    Therefore, for any
    $  \varphi(\cdot)=\varphi_{t_k}^{f_k}\circ \cdots
    \varphi_{t_2}^{f_2}\circ \varphi_{t_1}^{f_1}(\cdot)\in \mathcal A_{\mathcal F}$ with $ t_1+\cdots t_k\le T$ and $f_i\in \mathcal F(i=1,\cdots, k)$, we have
    \begin{equation}
        |\varphi(x)|_1\le e^{c_2(t_1+\cdots+t_k)}|x|_1+c_1\sum_{i=1}^{k}t_ie^{c_2(t_i+\cdots+t_k)}\le e^{c_2T}|x|_1+ c_1Te^{c_2T}.
    \end{equation}
Consequently, functions in $\mathcal A_{\mathcal F,T}$ are uniformly bounded on the compact set $K$.

Denote $L$ as the uniform Lipschitz constant of $\mathcal F$. For any $f\in\mathcal F$, we denote $x_1(t)$ and $x_2(t)$ as the solution of system~\eqref{eq:ode} with initial value $x_1(0)$ and $x_2(0)$, respectively. We then have
\begin{equation}
    \begin{aligned}
    |x_1(t)-x_2(t)|_1&=|x_1(0)-x_2(0)+\int_{0}^t (f(x_1(s)-x_2(s)))~\mathrm ds|_1\\
    &\le |x_1(0)-x_2(0)|_1+L\int_0^t |x_1(s)-x_2(s)|_1~\mathrm ds.
    \end{aligned}
\end{equation}
Using Gr\"onwall's inequality, we can get the following estimate holding for all $\varphi\in\mathcal A_{\mathcal F,T}$:
\begin{equation}
    |\varphi(x_1)-\varphi(x_2)|_1\le e^{LT} |x_1-x_2|_1.
\end{equation}
That means $\mathcal A_{\mathcal F,T}$ is uniformly Lipschitz.
\end{proof}
Now we can prove Proposition~\ref{prop:uni_time}.
\begin{proof}[Proof of Proposition~\ref{prop:uni_time}]

For any continuous $F\in \bar{\mathcal A}_{\mathcal F,T}$, there exists a sequence $\{\varphi_n\}_{n=1}^\infty
\in \mathcal A_{F,T}$  such that $\lim_{n\to \infty} \varphi_n=F$ in $L^p(K)$. Since $\mathcal A_{\mathcal F,T}$ is uniformly bounded and Lipschitz, and hence equicontinuous. By  Arzel\`a-Ascoli theorem, there exists a subsequence $\{\varphi_{n_k}\}$ converging to $F$ uniformly. Therefore, for any finite data $\mathcal D\sim (K,F)$, and tolerance $\varepsilon >0$, there exists some $\varphi_{n_k}$ such that it interpolates $\mathcal D$ with error less than $\varepsilon$. This implies the necessity.

Suppose any finite data $ \mathcal D\sim (K,F)$ can be approximately interpolated in time $T$. Since the set of points in $K$ with rational coordinates is countable but dense, we can enumerate these points as sequence $\{z_i\}_{i=1}^\infty$. By assumption, for any positive integer $n$, there exists $ \varphi_n\in\mathcal A_{\mathcal F,T}$ such that $|\varphi_n(z_i)-F(z_i)|<\frac{1}{n}$ for $i=1,2\cdots,n$. According to Arzel\`a-Ascoli theorem, there exists a uniformly convergent subsequence $\{\varphi_{n_k}\}$ of $\{\varphi_n\}$. Therefore, $\{\varphi_{n_k}\}$ uniformly converges to $F$ on a dense subset of $K$. Since $F$ is continuous, we deduce that $\{\varphi_{n_k}\}$ converges to $F$ uniformly on $K$. This implies the sufficiency.
\end{proof}

\bibliographystyle{siamplain}
\bibliography{ref}

\end{document}